\documentclass[review,11pt]{ReportTemplate}
\usepackage[colorlinks]{hyperref}
\hypersetup{citecolor=blue}
\usepackage{amsmath}
\usepackage{bm}
\usepackage{algorithm}
\usepackage{algorithmic}
\usepackage{textcomp,booktabs}
\usepackage{commath}
\usepackage{enumerate}
\usepackage{multicol}
\usepackage{graphicx} 
\usepackage{subfigure}
\usepackage{hyphenat}
\usepackage{wrapfig}

\newcommand{\x}{{\mathbf{x}}}
\newcommand{\w}{{\mathbf{w}}}
\newcommand{\wt}{{\mathbf{w}^\top}}
\newcommand{\X}{{\mathcal{X}}}
\newcommand{\W}{{\mathcal{W}}}

\newcommand{\s}{{\mathbf{s}}}
\newcommand{\p}{{\mathbf{p}}}

\newcommand{\ba}{{\boldsymbol{\alpha}}}
\newcommand{\bbeta}{{\boldsymbol{\beta}}}
\def \P  {{\mathcal{P}}}

\def \R {\mathbb{R}}
\newcommand{\E}{\mathop{\mathbb{E}}}

\newtheorem{Thm}{Theorem}
\newtheorem{Lemma}{Lemma}
\newenvironment{proof}{\textbf{Proof:}\ }{\hspace{\stretch{1}}$\square$}

\floatstyle{plain}
\newfloat{myalgo}{tbhp}{mya}
\newenvironment{Algorithm}[2][tbh]%
{\begin{myalgo}[#1]
\centering
\begin{minipage}{#2}
\begin{algorithm}[H]}%
{\end{algorithm}
\end{minipage}
\end{myalgo}}

\begin{document}
\begin{frontmatter}
\title{Top Rank Optimization in Linear Time}

\author{Nan Li$^{1}$}
\author{Rong Jin$^{2}$}
\author{Zhi-Hua Zhou$^1$\corref{cor1}}
\address{$^1$National Key Laboratory for Novel Software Technology,
Nanjing University, Nanjing 210023, China\\
$^2$Department of Computer Science and Engineering,
Michigan State University, East Lansing, MI 48824 \\}
\cortext[cor1]{\small Corresponding author. Email: zhouzh@lamda.nju.edu.cn}

\begin{abstract}
Bipartite ranking aims to learn a real-valued ranking function that orders positive instances before negative instances. Recent efforts of bipartite ranking are focused on optimizing ranking accuracy at the top of the ranked list. Most existing approaches are either to optimize task specific metrics or to extend the ranking loss by emphasizing more on the error associated with the top ranked instances, leading to a high computational cost that is super-linear in the number of training instances. We propose a highly efficient approach, titled {TopPush}, for optimizing accuracy at the top that has computational complexity linear in the number of training instances. We present a novel analysis that bounds the generalization error for the top ranked instances for the proposed approach. Empirical study shows that the proposed approach is highly competitive to the state-of-the-art approaches and is 10-100 times faster.
\end{abstract}


\end{frontmatter}

\section{Introduction}
Bipartite ranking aims to learn a real-valued ranking function that places positive instances above negative instances. It has attracted much attention because of its applications in several areas such as information retrieval and recommender systems~\citep{Rendle:2009:LOR,Liu11}. In the past decades, many ranking methods have been developed for bipartite ranking, and most of them are essentially based on pairwise ranking. These algorithms reduce the ranking problem into a binary classification problem by treating each positive-negative instance pair as a single object to be classified~\citep{Herbrich00,Freund-JMLR03,Burges-ICML05,ValizadeganJZM09,Usunier09,Rudin:2009,Agarwal11,BoydCMR12}.
Since the number of instance pairs can grow quadratically in the number of training instances, one limitation of these methods is their high computational costs, making them not scalable to large datasets.

Since for applications such as document retrieval and recommender systems,
only the top ranked instances will be examined by users, there has been a growing interest in learning ranking
functions that perform especially well at the top of the ranked list~\citep{Clemencon07,BoydCMR12}.
In the literature, most of these existing methods can be classified into two groups.
The first group maximizes the ranking accuracy at the top of the ranked list
by optimizing task specific metrics~\citep{Joachims05,Le07,Li:13,xu13}, such as average precision (AP)~\citep{Yue:2007},
NDCG~\citep{ValizadeganJZM09} and partial AUC~\citep{NarasimhanA-ICML13,NarasimhanA-KDD13}.
The main limitation of these methods is that they often result in non-convex optimization problems that are difficult to solve efficiently. Structural SVM~\citep{Tsochantaridis05} addresses this issue by translating the non-convexity into an exponential number of constraints. It can still be computationally challenging because it usually requires to search for the most violated constraint at each iteration of optimization. In addition, these methods are statistically inconsistency~\citep{Tewari:2007,Le07}, thus often leading to suboptimal solutions. The second group of methods are based on pairwise ranking. They design special convex loss functions that place more penalties on the ranking errors related to the top ranked instances, for example, by weighting~\citep{Usunier09} or exploiting special functions such as $p$-norm~\citep{Rudin:2009} and infinite norm~\citep{Agarwal11}. Since these methods are essentially based on pairwise ranking, their computational costs are usually proportional to the number of positive-negative instance pairs, making them unattractive for large datasets.

In this paper, we address the computational challenge of bipartite ranking by designing a ranking algorithm, named {TopPush}, that can efficiently optimize the ranking accuracy at the top. The key feature of the proposed TopPush algorithm is that its time complexity is only \emph{linear} in the number of training instances. This is in contrast to most existing methods for bipartite ranking whose computational costs depend on the number of instance pairs. Moreover, we develop novel analysis for bipartite ranking. One shortcoming of the existing theoretical studies~\citep{Rudin:2009,Agarwal11} on bipartite ranking is that they try to bound the probability for a positive instance to be ranked before \emph{any} negative instance, leading to relatively pessimistic bounds. We overcome this limitation by bounding the probability of ranking a positive instance before \emph{most} negative instances, and show that TopPush is effective in placing positive instances at the top of a ranked list. Extensive empirical study shows that TopPush is computationally more efficient than most ranking algorithms, and yields comparable performance as the state-of-the-art approaches that maximize the ranking accuracy at the top.

The rest of this paper is organized as follows.  Section~\ref{sec:related} introduces the preliminaries of bipartite ranking, and addresses the difference between AUC optimization and maximizing accuracy at the top. Section~\ref{sec:for} presents the proposed TopPush algorithm and its key theoretical properties. Section~\ref{sec:pf} gives proofs and technical details. Section~\ref{sec:exp} summarizes the empirical study, and Section~\ref{sec:con} concludes this work with future directions.

\section{Bipartite Ranking: AUC vs Accuracy at the Top}
\label{sec:related}
Let $\mathcal{X} = \{\x \in \mathbb{R}^d: \|\x\| \leq 1\}$ be the instance space. Let $S = S_+ \cup S_-$ be a set of training instances, where $S_+ = \{\x^+_i \in \mathcal{X} \}_{i=1}^m$ and $S_- = \{\x_i^- \in \mathcal{X}\}_{i=1}^n$ include $m$ positive instances and $n$ negative instances independently sampled from distributions $\P_+$ and $\P_-$, respectively. The goal of bipartite ranking is to learn a ranking function $f: \X\mapsto \mathbb{R}$ that is likely to place a positive instance before most negative ones. In the literature, bipartite ranking has found applications in many domains, and its theoretical properties have been examined by several studies~\citep[for example, ][]{Agarwal-JMLR05,Clemencon08,KotlowskiDH11,Narasimhan-NIPS13}.

AUC is a commonly used evaluation metric for bipartite ranking~\citep{Hanley82,CortesNIPS03}. By exploring its equivalence to Wilcoxon-Mann-Whitney statistic~\citep{Hanley82}, many ranking algorithms have been developed to optimize AUC by minimizing the ranking loss defined as
\begin{align}\label{eq:rl}
\mathcal{L}_{\rm rank}(f; S) = \frac{1}{mn}\sum_{i=1}^m \sum_{j=1}^n ~ \mathbb{I}\big(f(\x_i^+) \leq f(\x_j^-)\big) \ ,
\end{align}
where $\mathbb{I}(\cdot)$ is the indicator function with $\mathbb{I}(\texttt{true})=1$ and $0$ otherwise.  Other than a few special loss functions such as exponential and logistic loss~\citep{Rudin:2009,KotlowskiDH11}, most of these methods need to enumerate all the positive-negative instance pairs, making them unattractive for large datasets. Various methods have been developed to address this computational challenge. For example, in recent years, \cite{ZhaoHJY11} and \cite{Gao13} respectively studied online and one-pass AUC optimization .

\medskip

In recent literature, there is a growing interest in optimizing accuracy at the top of the ranked list~\citep{Clemencon07,BoydCMR12}. Maximizing AUC is not suitable for this goal as indicated by the analysis in~\citep{Clemencon07}. To address this challenge, we propose to maximize the number of positive instances that are ranked before the first negative instance, which is known as \emph{positives at the top}~\citep{Rudin:2009,Agarwal11,BoydCMR12}. We can translate this objective into the minimization of the following loss
\begin{align}\label{eq:ipl2}
\mathcal{L}(f; S) =  \frac{1}{m}\sum_{i=1}^m ~ \mathbb{I}\Big(f(\x_i^+) \leq \max_{1\leq j\leq n} f(\x_j^-)\Big) \ .
\end{align}
which computes the fraction of positive instances ranked below the top ranked negative instance.
By minimizing the loss in (\ref{eq:ipl2}), we essentially push negative instances away from
the top of the ranked list, leading to more positive ones placed at the top.
We note that (\ref{eq:ipl2}) is fundamentally different from AUC optimization as AUC does not focus on the ranking accuracy at the top. This can be seen from the relationship between the loss functions (\ref{eq:rl}) and (\ref{eq:ipl2}) as summarized below.
\begin{Prop}\label{thm:loss}
Let $S$ be a dataset consisting of $m$ positive instances and $n$ negative instances, and $f: \X\mapsto \mathbb{R}$ be a ranking function, we have
\begin{equation}
\mathcal{L}_{\rm rank}(f; S) \leq \mathcal{L}(f; S) \leq \min\big(n \mathcal{L}_{\rm rank}(f; S), 1\big) \ .
\end{equation}
\end{Prop}
The proof of this proposition is deferred to Section~\ref{sec:app-prop}.
According to Proportion~\ref{thm:loss}, we can see if the ranking loss $\mathcal{L}_{\rm rank}(f; S)$ is greater than $1/n$ which is common in practice, the loss $\mathcal{L}(f; S)$ can be as large as one, implying that no positive instance is ranked above any negative instance. Surely, this is not what we want, also it indicates that our goal of maximizing positives at the top can not be achieved by AUC optimization, consistent with the theoretical analysis in~\citep{Clemencon07}.
Meanwhile, we can find that $\mathcal{L}(f; S)$ is an upper bound over the ranking loss $\mathcal{L}_{\rm rank}(f; S)$, thus by minimizing $\mathcal{L}(f; S)$, small ranking loss can be expected, benefiting AUC optimization. This constitutes the main motivation of current work.

\smallskip

To design practical learning algorithms, we replace the indicator function in (\ref{eq:ipl2}) with its convex surrogate, leading to the following loss function
\begin{align}\label{eq:ipls}
\mathcal{L}^\ell(f; S) =  \frac{1}{m}\sum_{i=1}^m ~ \ell\Big(\max_{1\leq j\leq n} f(\x_j^-) - f(\x_i^+)\Big) \ ,
\end{align}
where $\ell(\cdot)$ is a convex surrogate loss function that is non-decreasing\footnote{ In this paper, we let $\ell(z)$ to be non-decreasing for the simplicity of formulating dual problem. } and differentiable. Examples of such loss functions include truncated quadratic loss $\ell(z) = [1+z]_+^2$, exponential loss $\ell(z) = e^{z}$, and logistic loss $\ell(z) = \log(1+e^z)$, etc. In the discussion below, we restrict ourselves to the truncated quadratic loss, even though most of our analysis applies to other loss functions.

It is easy to verify that the loss function $\mathcal{L}^\ell(f; S) $ in (\ref{eq:ipls}) is equivalent to the loss used in InfinitePush~\citep{Agarwal11} (a special case of $P$-norm Push~\citep{Rudin:2009})
\begin{align}\label{eq:ipl1}
\mathcal{L}^\ell_{\infty}(f; S) = \max_{1\leq j\leq n} ~  \frac{1}{m}\sum_{i=1}^m \ell\big(f(\x_j^-) - f(\x_i^+) \big) \ .
\end{align}
The apparent advantage of employing $\mathcal{L}^\ell(f; S)$ instead of $\mathcal{L}^\ell_{\infty}(f; S)$ is that it
only needs to evaluate on $m$ positive-negative instance pairs, whereas the later needs to enumerate all the $mn$ instance pairs. %
As a result, the number of dual variables induced by $\mathcal{L}^\ell(f; S)$ is $n+m$, linear in the number of training instances, which is significantly smaller than $mn$, the number of dual variables induced by $\mathcal{L}^\ell_{\infty}(f; S)$~\citep[see][]{Agarwal11,Rakotomamonjy12}. It is this difference that makes the proposed algorithm achieve a computational complexity linear in the number of training instances and therefore be more efficient than  most state-of-the-art algorithms for bipartite ranking.

\section{TopPush for Optimizing Top Accuracy}\label{sec:for}
In this section, we first present a learning algorithm to minimize the loss function in (\ref{eq:ipls}), and then the computational complexity and performance guarantee for the proposed algorithm.

\subsection{Dual Formulation}
We consider linear ranking function, that is  $f(\x) = \w^{\top}\x$, where $\w \in \mathbb{R}^d$ is the weight vector to be learned. {For nonlinear ranking function, we can use kernel methods, and Nystr\"{o}m method and random Fourier features can transform the kernelized problem into a linear one, see~\citep{YangLMJZ12} for more discussions on this topic.
As a result, the learning problem is given by the following optimization problem
\begin{align}\label{eq:ip}
\min_{\w} ~~ \frac{\lambda}{2}\|\w\|^2 + \frac{1}{m}\sum_{i=1}^m \ell\Big(\max_{1\leq j\leq n} \wt \x_j^--\wt\x_i^+\Big)\ ,
\end{align}
where $\lambda > 0$ is a regularization parameter.

Directly minimizing the objective in (\ref{eq:ip}) can be challenging because of the max operator in the loss function. We address this challenge by developing a dual formulation for (\ref{eq:ip}). Specifically, given a convex and differentiable function $\ell(z)$, we can rewrite it in its convex conjugate form as
\[
\ell(z) = \max_{\alpha \in \Omega} ~ \alpha z - \ell_*(\alpha) \ ,
\]
where $\ell_*(\alpha)$ is the convex conjugate of $\ell(z)$ and $\Omega$ is the domain of dual variable~\citep{bv-cvx}. For example, the convex conjugate of truncated quadratic loss is
\[
\ell_*(\alpha) = -\alpha + \alpha^2/4~~~\text{with}~~~\Omega = \R_+ \ .
\]
We note that dual form has been widely used to improve computational efficiency~\citep{Sun:2010} and connect different styles of learning algorithms~\citep{Kanamori:2013}. Here we exploit this technique to overcome the difficulty caused by max operator.
The dual form of (\ref{eq:ip}) is given in the following theorem, whose detailed proof is deferred to section~\ref{sec:dual}.
\begin{Thm} \label{thm:p}
Define $\mathbf{X}^+ = (\x_1^+, \ldots, \x_m^+)^{\top}$ and $\mathbf{X}^- = (\x_1^-, \ldots, \x_n^-)^{\top}$, the dual problem of the problem in (\ref{eq:ip}) is
\begin{align}\label{eq:op4}
\min_{(\ba, \bbeta) \in \Xi} \; g(\ba, \bbeta) = \frac{1}{2\lambda m} \|\ba^\top \mathbf{X}^+ - \bbeta^\top\mathbf{X}^-\|^2 + \sum_{i=1}^m\ell_*(\alpha_i)
\end{align}
where $\ba$ and $\bbeta$ are dual variables, and the domain $\Xi$ is defined as
\begin{align}
\Xi = \big\{\ba \in \R_+^m, \  \bbeta \in \R_+^n:\ \mathbf{1}_m^{\top}\ba = \mathbf{1}_n^{\top}\bbeta~\big\}.
\end{align}
Let $\ba^*$ and $\bbeta^*$ be the optimal solution to the dual problem in (\ref{eq:op4}). Then, the optimal solution $\w^*$ to the primal problem in (\ref{eq:ip}) is given by
\begin{equation}\label{eq:w}
\w^* = \frac{1}{\lambda m } \big(\bm{a}^{*\top}\mathbf{X}^+ - \bbeta^{*\top}\mathbf{X}^-\big)\ .
\end{equation}
\end{Thm}

The key feature of the dual problem in (\ref{eq:op4}) is that the number of dual variables is $m+n$. This is in contrast to the InfinitPush algorithm~\citep{Agarwal11} that introduces $mn$ dual variables.
In addition, the objective function in (\ref{eq:op4}) is smooth if the convex conjugate $\ell_*(\cdot)$ is smooth, which is true for many common loss functions (e.g., truncated quadratic loss, exponential loss and logistic loss). It is well known in the literature of optimization that an $O(1/T^2)$ convergence rate can be achieved if the objective function is smooth, where $T$ is the number of iterations. Surely, this also helps in designing efficient learning algorithm.

\subsection{Linear Time Bipartite Ranking Algorithm}
According to Theorem~\ref{thm:p}, to learn a ranking function $f(\w)$, it is sufficient to learn the dual variables $\ba$ and $\bbeta$ by solving the problem in (\ref{eq:op4}). For this purpose, we adopt the accelerated gradient method due to its light computation per iteration. Since we are pushing positive instances before the top-ranked negative,  we refer the obtained algorithm as \textbf{TopPush}.

\subsubsection{Efficient Optimization}
We choose the Nesterov's method~\citep{Nesterov03,Nemirovski94} that achieves an optimal convergence rate $O(1/T^2)$ for smooth objective function.
One of the key features of the Nesterov's  method is that besides the solution sequence $\{(\ba_k, \bbeta_k)\}$, it also maintains a sequence of auxiliary solutions $\{(\s^\alpha_k; \s^\beta_k)\}$, which is introduced to exploit the smoothness of the objective function to achieve faster convergence rate. Meanwhile, its step size depends on the smoothness of the objective function, in current work, we adopt the Nemirovski's line search scheme~\citep{Nemirovski94} to estimate the smoothness parameter. Of course, other schemes such as the one developed in~\citep{Liu:2009} can also be used.

Algorithm~\ref{algo:tp} summarizes the steps of the TopPush algorithm. At each iteration, the gradients of the objective function $g(\ba, \bbeta)$ can be efficiently computed as
\begin{align}\label{eq:grad}
\nabla_\ba g(\ba, \bbeta) = \frac{\mathbf{X}^+ \bm{\nu}^\top}{\lambda m} + \ell_*'(\ba)\ , ~~~~
\nabla_\bbeta g(\ba, \bbeta) = -\frac{\mathbf{X}^- \bm{\nu}^\top}{\lambda m} \ .
\end{align}
where $\bm{\nu} = \ba^\top \mathbf{X}^+ - \bbeta^\top\mathbf{X}^-$ and $\ell_*'(\cdot)$ is the derivative of $\ell_*(\cdot)$.
It should be noted that, the problem in (\ref{eq:op4}) is a constrained optimization problem, and therefore, at each step of gradient mapping, we have to project the dual solution into the domain $\Xi$ (that is, in step~\ref{step:proj}) to keep them feasible. Below, we discuss how to solve this projection step efficiently.

\begin{Algorithm}[!t]{12cm}
\caption{The TopPush Algorithm}\label{algo:tp}
\renewcommand\arraystretch{0.72}
\begin{algorithmic}[1]
\REQUIRE $\mathbf{X}^+ \in \mathbb{R}^{m\times d}$, $\mathbf{X}^- \in \mathbb{R}^{n\times d}$, $\lambda$, $\epsilon$
\ENSURE $\w$
\STATE let $t_{-1}=0$, $t_0=1$ and $L_0=\frac{1}{m+n}$
\STATE initialize $\ba_1 = \ba_0 = \bm{0}_m$ and $\bbeta_1 = \bbeta_0 = \bm{0}_n$
\FOR{$k = 0, 1, 2,\ldots$ }
\STATE set  $\omega_k = \frac{t_{k-2}-1}{t_{k-1}}$ and $L_k = L_{k-1}$
\STATE compute the auxiliary solution: \\
\quad $\s^a_k = \ba_k + \omega_k(\ba_k-\ba_{k-1})$ and  $\s^\beta_k = \bbeta_k + \omega_k(\bbeta_k-\bbeta_{k-1})$
\STATE compute the gradient at the auxiliary solution: \\
\quad $\mathbf{g}_\ba = \nabla_\ba g(\s^\alpha_k, \s^\beta_k)$ and  $\mathbf{g}_\bbeta = \nabla_\bbeta g(\s^\alpha_k, \s^\beta_k)$
\WHILE{true}
\STATE compute $\ba'_{k+1} = \s^\alpha_k - \frac{1}{L_k} \mathbf{g}_\ba$ and  $\bbeta'_{k+1} = \s^\beta_k - \frac{1}{L_k} \mathbf{g}_\bbeta$
\STATE {\sf Projection Step}: (by invoking Algorithm~\ref{algo:p}) \label{step:proj}\\
\quad $[\ba_{k+1}; \bbeta_{k+1}] = {\pi}_{\Xi}([\ba'_{k+1}; \bbeta'_{k+1}])$
\IF{$g(\ba_{k+1}, \bbeta_{k+1}) \leq g(\s^\alpha_{k}, \s^\beta_{k}) + \frac{\|\mathbf{g}_\ba\|^2 + \|\mathbf{g}_\bbeta\|^2}{2L_k}$}
\STATE break
\ENDIF
\STATE $L_k=2L_k$
\ENDWHILE
\STATE update $t_k = (1+\sqrt{1+4t^2_{k-1}})/2$
\IF{$|g(\ba_{k+1},\bbeta_{k+1}) - g(\ba_{k},\bbeta_{k})| < \epsilon$}
\RETURN $\w = \frac{1}{\lambda \cdot m } (\ba_{k+1}^{\top}\mathbf{X}^+ - \bbeta_{k+1}^{\top}\mathbf{X}^-)$
\ENDIF
\ENDFOR
\end{algorithmic}
\end{Algorithm}

\subsubsection{Projection Step}
For clear notations, we expand the projection step into the problem
\begin{align}\label{eq:pp}
\min_{\ba\geq 0, \bbeta\geq 0} & ~~ \frac{1}{2}\|\ba - \ba^0\|^2 + \frac{1}{2}\|\bbeta - \bbeta^0\|^2 \\
\text{s.t.}& ~~\bm{1}_m^\top \ba = \bm{1}_n^\top\bbeta   \nonumber
\end{align}%
where $\ba^0$ and $\bbeta^0$ are the solutions to be projected. We note that similar projection
problems have been studied in~\citep{Shalev-Shwartz:2006,Liu-ICML09} whereas they either have $O((m+n)\log(m+n))$
time complexity or only provide approximate solutions.
Instead, based on the following proposition, we provide a method which find the {\it exact} solution to~(\ref{eq:pp})
in $O(n+m)$ time.

\begin{Prop}\label{thm:proj}
The optimal solution to the projection problem in (\ref{eq:pp}) is given by
\[
\ba^* = \big[\ba^0 - \gamma^*\big]_+~~~ \text{and}~~~ \bbeta^* = \big[\bbeta^0 +\gamma^*\big]_+ \ ,
\]
where $\gamma^*$ is the unique root of function 
\begin{equation}\label{eq:af}
\rho(\gamma) = \sum_{i=1}^m \big[\alpha_i^0-\gamma\big]_+ -\sum_{j=1}^n \big[\beta_j^0+\gamma\big]_+ \ .
\end{equation}
\end{Prop}
The proof of this proposition is similar to that for~\citep[Theorem 2]{Liu-ICML09}, thus omitted here.
According to Proposition~\ref{thm:proj}, the key to solving the projection problem is to find the root of $\rho(\gamma)$.
Instead of approximating the solution via bisection as in~\citep{Liu-ICML09}, we develop a different scheme to get the exact solution as follows.

For a given value of $\gamma$, define two index sets
\begin{align*}
\mathcal{I}(\gamma) = \big\{i \in [1, m] : \alpha_i^0 > \gamma \big\}\  ~~~\text{and}~~~
\mathcal{J}(\gamma) = \big\{j\in[1,n] : \beta_j^0 \geq -\gamma \big\}\ ,
\end{align*}
then the function $\rho(\gamma)$ in (\ref{eq:af}) can be rewrite as
\begin{equation}\label{eq:af2}
\rho(\gamma) = \sum_{i \in \mathcal{I}(\gamma)} \alpha_i^0 -\sum_{j \in \mathcal{J}(\gamma)} \beta_j^0 - \big(|\mathcal{I}(\gamma)| + |\mathcal{J}(\gamma)|\big) \gamma \ .
\end{equation}
Also, define
\[
\mathcal{U} = \{ \alpha^0_i : 1\leq i\leq m\} \cup\{-\beta^0_j : 1\leq j\leq n\} \ ,
\]
and let $u_{(i)}$ denote its $i$-th order statistics, that is, $u_{(1)} \leq u_{(2)} \leq \ldots, \leq u_{(|\mathcal{U}|)}$.
It can be found that for a given $k$ and any $\gamma$ in the interval $[u_{(k)}, u_{(k+1)})$, it holds that
\[
\mathcal{I}(\gamma)=\mathcal{I}(u_{(k)})~~\text{and}~~\mathcal{J}(\gamma)=\mathcal{J}(u_{(k)})\ .
\]
Thus, from (\ref{eq:af2}), if the interval $[u_{(k)}, u_{(k+1)})$ contains the root of $\rho(\gamma)$, the root $\gamma^*$ can be exactly computed as
\begin{equation}\label{eq:gs}
\gamma^* = \frac{\sum_{i \in \mathcal{I}(u_{(k)})} \alpha_i^0 -\sum_{j \in \mathcal{J}(u_{(k)})} \beta_j^0}{|\mathcal{I}(u_{(k)})| + |\mathcal{J}(u_{(k)})|} \ .
\end{equation}
Consequently, the task can be reduced to finding $k$ such that $\rho(s_{(k)}) > 0$ and $\rho(s_{(k+1)}) \leq 0$.

Inspired by~\citep{DuchiSSC08}, we devise a divide-and-conquer procedure based on a modification of the randomized median finding algorithm~\citep[Chapter 9]{Cormen01}, and it is summarized in Algorithm~\ref{algo:p}.
In particular, it maintains a set\footnote{To make the updating of partial sums efficient, in practice, two sets $\mathcal{U}^\ba$ and $\mathcal{U}^\bbeta$ are respectively maintained for $\ba^0$ and $-\bbeta^0$, and $\mathcal{U}$ is their union. Also, the sets $\mathcal{G}$ and $\mathcal{L}$ are handled in a similar manner. } of unprocessed elements from $\mathcal{U}$, whose relationship to an element $u$ we do not know. On each round, we partition $\mathcal{U}$ into two subsets $\mathcal{G}$ and $\mathcal{L}$, which respectively contains the elements in $\mathcal{U}$ that are respectively greater and less than the element $u$ that is picked up at random from $\mathcal{U}$.
Then, by evaluating the function $\rho$ in (\ref{eq:af2}), we update $\mathcal{U}$ to the set (i.e., $\mathcal{G}$ or $\mathcal{L}$) containing the needed element and discard the other. The process ends when $\mathcal{U}$ is empty. Afterwards, we compute the exact optimal $\gamma^*$ as (\ref{eq:gs}) and perform projection as described in Proposition~\ref{thm:proj}.
In addition, for efficiency issues, along the process we keep track of the partial sums in (\ref{eq:af2}) such that they will be not recalculated.
Based on similar analysis of the randomized median finding algorithm, we can obtain Algorithm~\ref{algo:p} has expected linear time complexity.

\begin{Algorithm}[!t]{12cm}
\caption{Linear Time Projection}\label{algo:p}
\begin{algorithmic}[1]
\REQUIRE $\ba^0 \in \mathbb{R}^{m}$, $\bbeta^0 \in \mathbb{R}^{n}$
\ENSURE $\ba^*$, $\bbeta^*$
\STATE initialize $\mathcal{U}^\alpha = \{\alpha^0_i\}_{i=1}^m$, $\mathcal{U}^\beta = \{-\beta^0_j\}_{j=1}^n$,  and $\mathcal{U} = \mathcal{U}^\alpha  \cup \mathcal{U}^\beta $
\STATE initialize $s^\alpha=0$, $s^\beta=0$, $n^\alpha=0$, $n^\beta=0$
\WHILE{$\mathcal{U} \neq \varnothing$}
\STATE pick $u \in \mathcal{U}$ at random, and use it to partition $\mathcal{U}^a $ and $\mathcal{U}^q $: \\
\hspace{6mm} $\mathcal{G}^\alpha = \{\alpha \in \mathcal{U}^\alpha : \alpha > u\}$
\hspace{6mm} $\mathcal{L}^\alpha = \{\alpha \in \mathcal{U}^\alpha : \alpha \leq u\}$  \\
\hspace{6mm} $\mathcal{G}^\beta = \{\beta \in \mathcal{U}^\beta : \beta \geq u\}$
\hspace{6.5mm} $\mathcal{L}^\beta = \{\beta \in \mathcal{U}^\beta : \beta < u\}$
\STATE compute
$\Delta n^\alpha = |\mathcal{G}^\alpha|$, $\Delta s^\alpha = \sum_{\alpha \in \mathcal{G}^\alpha} \alpha$ and \\
\hspace{14.5mm} $\Delta n^\beta = |\mathcal{L}^\beta|$, $\Delta s^\beta = \sum_{\beta \in \mathcal{L}^\beta} \beta$
\STATE let $s' =s^\alpha +\Delta s^\alpha + s^\beta + \Delta s^\beta$ and  $n' = n^\alpha +\Delta n^\alpha + n^\beta + \Delta n^\beta$
\IF {$s' < n' u$}
\STATE update $\mathcal{U}^\alpha= \mathcal{L}^\alpha$ and $\mathcal{U}^\beta= \mathcal{L}^\beta$\\
\STATE update $s^\alpha = s^\alpha + \Delta s^\alpha$ and $n^\alpha = n^\alpha + \Delta n^\alpha$
\ELSE
\STATE update $\mathcal{U}^\alpha= \mathcal{G}^\alpha$ and $\mathcal{U}^\beta= \mathcal{G}^\beta$ \\
\STATE update $s^\beta = s^\beta + \Delta s^\beta$ and $n^\beta = n^\alpha + \Delta n^\beta$
\ENDIF
\STATE let $\mathcal{U} = (\mathcal{U}^\alpha  \cup \mathcal{U}^\beta) \setminus \{u\}$
\ENDWHILE
\STATE let $\gamma  ={(s^\alpha + s^\beta)}/{(n^\alpha + n^\beta)}$
\STATE return $\ba^* = \big[\ba - \gamma\big]_+$ and $\bbeta^* = \big[\bbeta^0 + \gamma\big]_+$
\end{algorithmic}
\end{Algorithm}

\subsection{Convergence and Computational Complexity}\label{sec:complexity}
The theorem below states the convergence of the TopPush algorithm, which follows immediately from the convergence result for the Nesterov's method~\citep{Nemirovski94}.

\begin{Thm}\label{thm:scale}
Let $\ba_T$ and $\bbeta_T$ be the solution output from the TopPush algorithm after $T$ iterations, we have
\[
g(\ba_T, \bbeta_T) \leq \min\limits_{(\ba, \bbeta) \in \Xi} g(\ba, \bbeta) + \epsilon
\]
provided $T \geq O(1/\sqrt{\epsilon})$.
\end{Thm}

Finally, the computational cost of each iteration is dominated by the gradient evaluation and the projection step. Since the complexity of projection step is $O(m+n)$ and the cost of computing the gradient is $O((m+n)d)$, the time complexity of each iteration is $O((m + n)d)$.  Combining this result with Theorem~\ref{thm:scale}, we have, to find an $\epsilon$-suboptimal solution, the total computational complexity of the TopPush algorithm is $O((m + n)d/\sqrt{\epsilon})$, which is linear in the number of training instances.

\begin{table}[t]
\centering
\small
\begin{tabular}{lll}
\toprule
  \textbf{Algorithm} &  & \textbf{Computational Complexity} \\
\midrule
  SVM$^{\rm Rank}$ & \citep{Joachims:2006} & $O(((m+n)d + (m+n)\log (m+n))/\epsilon)$ \\
  SVM$^{\rm MAP}$ & \citep{Yue:2007}& $O(((m+n)d +(m+n)\log (m+n))/\epsilon)$\\
  OWPC & \citep{Usunier09} & $O(((m+n)d +(m+n)\log (m+n))/\epsilon)$\\
  SVM$^{\rm pAUC}$ & \citep{NarasimhanA-ICML13,NarasimhanA-KDD13} & $O((n\log n + m\log m + (m+n)d)/\epsilon)$ \\
  InfinitePush & \citep{Agarwal11} & $O((mnd + mn\log(mn))/\epsilon^2)$\\
  L1SVIP & \citep{Rakotomamonjy12} & $O((mnd + mn\log(mn))/\epsilon)$\\
  TopPush & this paper & $O((m+n)d/\sqrt{\epsilon})$\\
\bottomrule
\end{tabular}
\caption{Comparison of computational complexities for ranking algorithms, where $m$ and $n$ are the number of positive and negative instances, $d$ is the number of dimensions, and $\epsilon$ is the precision parameter. }\label{tbl:method}
\end{table}

\medskip

Table~\ref{tbl:method} compares the computational complexity of TopPush with that of some state-of-the-art ranking algorithms. It is easy to see that TopPush is asymptotically more efficient than the state-of-the-art ranking algorithm\footnote{In Table~\ref{tbl:method}, we report the complexity of SVM$_{\rm tight}^{\rm pAUC}$ in~\citep{NarasimhanA-KDD13}, which is more efficient than SVM$^{\rm pAUC}$ in~\citep{NarasimhanA-ICML13}. In addition, SVM$_{\rm tight}^{\rm pAUC}$ is used in experiments and we do not distinguish between them in this paper. }. For instances, it is much more efficient than InfinitePush and its sparse extension L1SVIP whose complexity depends on the number of positive-negative instance pairs; compared with SVM$^{\rm Rank}$, SVM$^{\rm MAP}$ and SVM$^{\rm pAUC}$ that handle specific performance metrics via structural-SVM, the linear dependence on the number of training instances makes our proposed TopPush algorithm more appealing, especially for large datasets.


\subsection{Theoretical Guarantee }\label{sec:theo}
We develop theoretical guarantee for the ranking performance of TopPush. 
In~\citep{Rudin:2009,Agarwal11}, the authors have developed margin-based generalization bounds
for the loss function $\mathcal{L}_{\infty}^\ell$ . One limitation with the analysis in~\citep{Rudin:2009,Agarwal11} is that they try to bound the probability for a positive instance to be ranked before \emph{any} negative instance, leading to relatively pessimistic bounds. {For instance, for the bounds in~\citep[Theorems 2 and 3]{Rudin:2009}, the failure probability can be as large as 1 if the parameter $p$ is large.  Our analysis avoids this pitfall by considering the probability of ranking a positive instance before \emph{most} negative instances.

To this end, we first define $h_b(\x, \w)$, the probability for any negative instance to be ranked above $\x$ using ranking function $f(\x) = \w^{\top}\x$, as
\begin{align*}
h_b(\x, \w) & =  \E_{\x^-\sim \P^-} \big[ \mathbb{I}(\w^{\top}\x \leq \w^{\top}\x^-) \big] \ .
\end{align*}
Since we are interested in whether positive instances are ranked above {\it most} negative instances, we will measure the quality of $f(\x) = \w^{\top}\x$ by the probability for any positive instance to be ranked below $\delta$ percent of negative instances, that is
\begin{align*}
P_b(\w, \delta) &= \Pr_{\x^+\sim\P^+}\big(h_b(\x_i^+, \w) \geq \delta\big) \ .
\end{align*}
Clearly, if a ranking function achieves a high ranking accuracy at the top, it should have a large percentage of positive instances with ranking scores higher than most of the negative instances, leading to a small value for $P_b(\w, \delta)$ with little $\delta$. The following theorem bounds $P_b(\w, \delta)$ for TopPush, whose proof can be found in the supplementary document.

\begin{Thm}\label{thm:bound}
Given training data $S$ consisting of $m$ independent samples from $\P^+$ and $n$ independent samples from $\P^-$, let $\w^*$ be the optimal solution to the problem in (\ref{eq:ip}).
Assume $m \geq 12$ and $n\gg t$, we have, with a probability at least $1- 2e^{-t}$,
\[
P_b(\w^*, \delta) \leq \mathcal{L}^\ell(\w^*, S) + O\big(\sqrt{{(t + \log m)}/{m}}\big)
\]
where $\delta = O(\sqrt{\log m/{n}})$ and $$\mathcal{L}^\ell(\w^*, S) = \frac{1}{m}\sum_{i=1}^m \ell(\max_{1\leq j\leq n} \w^{*\top} \x_j^--\w^{*\top}\x_i^+)$$ is the empirical loss.
\end{Thm}

Theorem~\ref{thm:bound} implies that if the empirical loss $\mathcal{L}^\ell(\w^*, S) \leq O(\log m / m)$, for most positive instance $\x_+$ (i.e., $1-O(\log m / m)$), the percentage of negative instances ranked above $\x_+$ is upper bounded by $O(\sqrt{\log m/n})$.
We observe that $m$ and $n$ play different roles in the bound. That is, since the empirical loss compares the positive instances to the negative instance with the largest score, it usually grows significantly slower with increasing $n$. For instance, the largest absolute value of Gaussian random samples grows in $\log n$. Thus, we believe that the main effect of increasing $n$ in our bound is to reduce $\delta$ (decrease at the rate of $1/\sqrt{n}$), especially when $n$ is large.
Meanwhile, by increasing the number of positive instances $m$, we will reduce the bound for $P_b(\w, \delta)$, and consequently increase the chance of finding positive instances at the top.

\section{Proofs and Technical Details}\label{sec:pf}
In this section, we give all the detailed proofs missing from the main text, along with ancillary remarks and comments.

\subsection{AUC vs. Accuracy at the Top}\label{sec:app-prop}
We investigate the relationship between AUC and accuracy at the top by their corresponding loss functions, i.e. the ranking loss $\mathcal{L}_{\rm rank}$ in~(\ref{eq:rl}) and our loss $\mathcal{L}$ in (\ref{eq:ipl2}). \medskip
\begin{proof}[of Proposition~\ref{thm:loss}]
It is easy to verify that the loss $\mathcal{L}$ in (\ref{eq:ipl2}) is equivalent to
\[
\mathcal{L}_{\infty}(f; S) =  \max_{1\leq j\leq n} ~  \frac{1}{m}\sum_{i=1}^m \mathbb{I}\big(f(\x_i^+) \leq f(\x_j^-)\big) \ .
\]
Define $\kappa_j = \frac{1}{m}\sum_{i=1}^m \mathbb{I}\big(f(\x_i^+) \leq f(\x_j^-)\big)$, thus we have $\kappa_j \in [0, 1]$, and
\[
\mathcal{L}(f; S) = \mathcal{L}_{\infty}(f; S) = \max_{1\leq j\leq n} \kappa_j  \ , \quad
\mathcal{L}_{\rm rank}(f; S) = \frac{1}{n} \sum\nolimits_{j=1}^n\kappa_j \ .
\]
Based on the relationship between the mean and the maximum of a set of elements, we can obtain the conclusion.
\end{proof}

\subsection{Proof of Theorem~\ref{thm:p}}\label{sec:dual}
Since $\ell(z)$ is a convex loss function that is non-decreasing and differentiable, it can be rewritten in its convex conjugate form, that is
\[
\ell(z) = \max_{\alpha \geq 0} ~ \alpha z - \ell_*(\alpha) \
\]
where $\ell_*(\alpha)$ is the convex conjugate of $\ell(z)$, and hence rewritten the problem in (\ref{eq:ip}) as
\begin{align}\label{eq:op1}
 \min_\w~\max_{\ba \geq 0}~~~ \frac{1}{m}\sum_{i=1}^m \alpha_i \Big(\max_{1\leq j\leq n} \wt\x_j^- - \wt\x_i^+\Big)
 - \frac{1}{m}\sum_{i=1}^m \ell_*(\alpha_i)  + \frac{\lambda}{2}\|\w\|^2 \ ,
\end{align}
where $\ba=(\alpha_1, \ldots, \alpha_m)^\top$ are dual variables.

Let $\p \in \mathbb{R}^n$ and $\Delta=\{\p: \p \geq 0~\text{and}~\bm{1}_n^\top\p=1\}$ be the standard $n$-simplex, we have
\begin{equation}\label{eq:max}
\max_{1\leq j\leq n} ~ \wt \x_j^- = \max_{\mathbf{p} \in \Delta} ~ \sum_{j=1}^n p_j \wt \x_j^- \ .
\end{equation}
By substituting (\ref{eq:max}) into (\ref{eq:op1}), the optimization problem  becomes
\begin{align}\label{eq:op2}
& \min_\w \max_{\ba \geq 0, \mathbf{p} \in \Delta} \frac{1}{m} \sum_{j=1}^np_j  \sum_{i=1}^m \alpha_i  \wt\x_j^-
~~- \frac{1}{m} \sum_{i=1}^m \alpha_i \wt\x_i^+
 - \frac{1}{m}\sum_{i=1}^m \ell_*(\alpha_i) + \frac{\lambda}{2} \|\w\|^2  .
\end{align}
By defining $\beta_j = p_j \sum_{i=1}^m \alpha_i $ and then using variable replacement, (\ref{eq:op2}) can be equivalently rewritten as
\begin{align}\label{eq:op3}
\nonumber\min_\w \max_{\ba\geq 0, \bbeta \geq 0} ~  &\frac{1}{m} \left(\sum_{j=1}^n \beta_j \wt\x_j^- - \sum_{i=1}^m \alpha_i \wt\x_i^+\right)  - \frac{1}{m}\sum_{i=1}^m \ell_*(\alpha_i) + \frac{\lambda}{2} \|\w\|^2 \\
 \text{s.t.}~~~ & \bm{1}_m^\top \ba = \bm{1}_n^\top \bbeta\ ,
\end{align}
where $\bbeta=[\beta_1, \ldots, \beta_n]^\top$ are new variables, the constraint $\mathbf{p} \in \Delta$ is replaced with the $\bbeta \geq 0$, and the equality constraint $\bm{1}_m^\top \ba = \bm{1}_n^\top \bbeta$ to keep two problems equivalent.

Since the objective of (\ref{eq:op3}) is convex in $\w$, and jointly concave in $\ba$ and $\bbeta$, also its feasible domain is convex; hence it satisfies the {strong max-min property}~\citep{bv-cvx}, the min and max can be swapped. After swapping min and max, we first consider the inner minimization subproblem over $\w$, that is
\begin{equation*}
\min_\w ~ \frac{1}{m}\sum_{j=1}^n \beta_j \wt\x_j^- - \frac{1}{m}\sum_{i=1}^m \alpha_i \wt\x_i^+  + \frac{\lambda}{2} \|\w\|^2 \ ,
\end{equation*}
where $\frac{1}{m}\sum_{i=1}^m\ell_*(a_i)$ is omitted since it does not depend on $\w$.
This is an unconstrained quadratic programming problem, whose solution is
\[
\w^* = \frac{1}{\lambda m } (\bm{a}^{\top}\mathbf{X}^+ - \bbeta^{\top}\mathbf{X}^-) \ ,
\]
and the minimal value is given as
\[
- \frac{1}{2\lambda m^2} \|\bm{a}^\top \mathbf{X}^+ - \bbeta^\top\mathbf{X}^-\|^2\ .
\]
Then, by considering the maximization over $\ba$ and $\bbeta$, we can obtain the conclusion of Theorem~\ref{thm:p} (after multiplying the objective function with $m$). \hfill  $\square$

\subsection{Proof of Theorem~\ref{thm:bound}}\label{sec:bound}
For the convenience of analysis, we consider the constrained version of the optimization problem in (\ref{eq:ip}), that is
\begin{eqnarray}
\min_{\w \in \W} \mathcal{L}^\ell(\w; S) = \frac{1}{m}\sum_{i=1}^m \ell\left(\max_{1\leq j\leq n} \wt \x_j^--\wt\x_i^+\right) \label{eq:ip-constraint}
\end{eqnarray}
where $\W = \{\w \in \mathbb{R}^d : \|\w\| \leq \rho\}$ is a domain and $\rho > 0$ specifies the size of the domain that plays similar role as the regularization parameter $\lambda$ in~(\ref{eq:ip}).

First, we denote $G$ as the Lipschitz constant of the truncated quadratic loss  $\ell(z)$ on the domain $[-2\rho, 2\rho]$, and define the following two functions based on $\ell(z)$, i.e.,
\begin{align*}
h_\ell(\x, \w) = \E_{\x^-\sim \P^-}\left[ \ell(\w^{\top}\x^- - \w^{\top}\x) \right] ~~~\text{and}~~~
P_\ell(\w, \delta) = \Pr_{\x^+\sim\P^+}\left(h_\ell(\x_i^+, \w) \geq \delta\right) \ .
\end{align*}
The lemma below relates the empirical counterpart of $P_\ell$ with the loss $\mathcal{L}^\ell$.
\begin{Lemma}\label{lemma-1}
With a probability at least $1 - e^{-t}$, for any $\w \in \W$, we have
\[
\frac{1}{m}\sum_{i=1}^m \mathbb{I}\left(h_\ell(\x_i^+, \w)\geq {\delta}\right) \leq \mathcal{L}^\ell(\w, S) \ ,
\]
where
\begin{equation}\label{eq:pf-delta}
\delta =
\frac{4G(\rho+1)}{\sqrt{n}} + \frac{5\rho(t + \log m)}{3n} +  2G\rho\sqrt{\frac{2(t + \log m)}{n}} \ .
\end{equation}
\end{Lemma}
\begin{proof}
For any $\w \in \W$, we define two instance sets by splitting $S^+$, that is
\begin{align*}
\mathcal{A}(\w) = \Big\{\x_i^+: \w^{\top}\x_i^+ > \max_{j\in[n]} \w^{\top}\x_j^- + 1\Big\}
\ , \ \
\mathcal{B}(\w) = \Big\{\x_i^+: \w^{\top}\x_i^+ \leq \max_{j\in[n]} \w^{\top}\x_j^- + 1\Big\}\ .
\end{align*}
For $\x_i^+ \in \mathcal{A}(\w)$, we define
\begin{align*}
\|P &- P_n\|_{\W} = \sup_{\|\w\| \leq \rho} \left| h_{\ell}(\x_i^+, \w) - \frac{1}{n}\sum_{j=1}^n \ell(\w^{\top}\x_j^- -\w^{\top}\x_i^+)\right| \ .
\end{align*}
Using the Talagrand's inequality and in particular its variant (specifically, Bousquet bound) with improved constants derived in~\citep{Bousquet024} \citep[see also][Chapter~2]{Koltchinskii11}, we have, with probability at least $1-e^{-t}$,
\begin{equation}\label{eq:talagrand}
\|P - P_n\|_{\W} \leq \E\|P - P_n\|_{\W} + \frac{2t\rho}{3n} + \sqrt{\frac{2t}{n}\left(\sigma^2_P(\W) + 2\mathbb{E}\|P - P_n\|_{\W}\right)}\ .
\end{equation}
We now bound each item on the right hand side of (\ref{eq:talagrand}). First, we bound $\mathbb{E}\|P_n - P\|_{\W}$ as
\begin{align}\label{eq:pf-B1}
\mathbb{E}\|P - P_n\|_{\W}
&= \frac{2}{n} \E\left[\sup_{\|\w\| \leq \rho} \sum_{j=1}^n \sigma_j \ell(\w^{\top}(\x_j^- -\x_i^+))\right] \nonumber \\
&\leq \frac{4G}{n} \E\left[\sup_{\|\w\| \leq \rho} \sum_{j=1}^n \sigma_j (\w^{\top}(\x_j^- -\x_i^+))\right]
\leq \frac{4G\rho}{\sqrt{n}} \ ,
\end{align}
where $\sigma_j$'s are Rademacher random variables, the fist inequality utilizes the contraction property of Rademacher complexity, and the last follows from Cauchy-Schwarz inequality and Jensen's inequality. Next,  we bound $\sigma^2_P(\W)$,  that is,
\begin{equation}\label{eq:pf-B2}
\sigma^2_P(\W) = \sup_{\|\w\| \leq \rho} h_\ell^2(\x, \w) \leq 4G^2\rho^2 \ .
\end{equation}
By putting (\ref{eq:pf-B1}) and (\ref{eq:pf-B2}) into (\ref{eq:talagrand}) and using the fact that
\[
\frac{1}{n}\sum_{j=1}^n \ell(\w^{\top}(\x_j^--\x_i^+)) = 0 ~~~\text{for}~~~ \x_i^+ \in \mathcal{A}(\w),
\]
we thus have, with probability $1-e^{-t}$,
\begin{align*}
|h_{\ell}(\x_i^+, \w)|
&\leq \|P - P_n\|_{\W}
\leq \frac{4G\rho}{\sqrt{n}} +  \frac{2t\rho}{3n} + \sqrt{\frac{2t}{n}\left( 4G^2\rho^2 + \frac{8G\rho}{\sqrt{n}}\right)} \\
& \leq \frac{4G\rho}{\sqrt{n}} + \frac{2t\rho}{3n} + 2G\rho\sqrt{\frac{2t}{n}} + \frac{4G}{\sqrt{n}} + \frac{t\rho}{n}\\
& \leq \frac{4G(\rho+1)}{\sqrt{n}} + \frac{5t\rho}{3n} +  2G\rho\sqrt{\frac{2t}{n}} \ .
\end{align*}
Using the union bound over all $\x_i^+$'s, we obtain
\[
\max_{\x_i^+ \in \mathcal{A}(\w)} h_\ell(\x_i^+, \w) \leq \delta\ ,
\]
where $\delta$ is in (\ref{eq:pf-delta}). Thus, with probability $1-e^{-t}$, it follows
\[
\sum_{\x_i^+ \in \mathcal{A}(\w)} \mathbb{I}\left(h_\ell(\x_i^+, \w)\geq {\delta}\right) = 0\ .
\]
Therefore, we can obtain the conclusion based on the fact $|\mathcal{B}(\w)| \leq m \mathcal{L}^\ell(\w, S)$.
\end{proof}

Based on Lemma~\ref{lemma-1}, we are at the position to prove Theorem~\ref{thm:bound}.

\medskip

\begin{proof}[of Theorem~\ref{thm:bound}]
Let $\mathcal{S}(\W, \varepsilon)$ be a proper $\varepsilon$-net of $\W$ and $N(\rho, \varepsilon)$ be the corresponding covering number. According to standard result, we have
\[
\log N(\rho, \varepsilon) \leq d\log (9\rho / \varepsilon)\ .
\]
By using concentration inequality and union bound over $\w' \in \mathcal{S}(\W, \varepsilon)$, we have, with probability at least $1-e^{-t}$,
\begin{equation}\label{eq:l2}
\sup_{\w' \in \mathcal{S}(\W, \varepsilon)} P_\ell(\w', \delta) - \frac{1}{m}\sum_{i=1}^m \mathbb{I}(h_\ell(\x_i^+, \w')\geq \delta)
\leq \sqrt{\frac{2(t+d\log(9\rho/\varepsilon))}{m}} \ \ .
\end{equation}
Let $\mathbf{d} = \x^- - \x^+$ and $\varepsilon = \frac{1}{2\sqrt{m}}$. For $\w^* \in \W$, there exists $\w' \in \mathcal{S}(\W, \varepsilon)$ such that $\|\w' -\w^*\| \leq \varepsilon$, it holds that
\begin{align*}
\mathbb{I}(\w^{*\top}& \mathbf{d} \geq 0) = \mathbb{I}(\w'^\top \mathbf{d} \geq (\w'-\w^*)^\top \mathbf{d})
\leq  \mathbb{I}(\w'^\top \mathbf{d} \geq -\frac{1}{\sqrt{m}}) \leq 2 \ell(\w'^\top \mathbf{d}) \ .
\end{align*}
where the last step is based on $\ell(\cdot)$ is non-decreasing and
$
\ell(-{1}/{\sqrt{m}}) \geq \frac{1}{2} ~ \text{if} ~ m \geq 12 \ .
$
We thus have $h_b(\x^+, \w^*) \leq 2 h_\ell(\x^+, \w')$ and therefore $P_b(\w^*, \delta) \leq P_\ell(\w', {\delta}/{2})$.

As a consequence, from (\ref{eq:l2}), Lemma~\ref{lemma-1} and the fact
\[
\mathcal{L}^\ell_{k}(\w', S) \leq  \mathcal{L}^\ell_{k}(\w, S) +  \frac{G\rho}{\sqrt{m}}\ ,
\]
we have, with probability at least $1-2e^{-t}$,
\[
P_b(\w^*, \delta) \leq \mathcal{L}^\ell_{k}(\w^*, S) + \frac{G\rho}{\sqrt{m}} + \sqrt{\frac{2t+2d\log(9\rho) + d\log m}{m}} \  \ ,
\]
where $\delta$ is as defined in (\ref{eq:pf-delta}), and the conclusion follows by hiding constants.
\end{proof}

\section{Experiments}\label{sec:exp}
To evaluate the performance of the proposed TopPush algorithm, we conduct a set of experiments on real-world datasets.
\subsection{Settings}
Table~\ref{tbl:data} (left column) summarizes the datasets used in our experiments. Some of them were used in previous studies~\citep{Agarwal11,Rakotomamonjy12,BoydCMR12}, and others are larger datasets from different domains. For example, \texttt{diabetes} is a medical task, \texttt{news20-forsale} is on text classification, \texttt{spambase} is about email spam filtering, and \texttt{nslkdd} is a network intrusion dataset. It should be noted that \texttt{news20-forsale} is transformed from the \texttt{news20} dataset by treating \texttt{forsale} as positive class and others as negative. All these datasets are publicly available\footnote{These datasets are available at \url{http://www.csie.ntu.edu.tw/~cjlin/libsvmtools/datasets} and \url{http://nsl.cs.unb.ca/NSL-KDD/} \ . }.

We compare TopPush with state-of-the-art ranking algorithms that focus on accuracy at the top, including SVM$^{\rm MAP}$~\citep{Yue:2007}, SVM$^{\rm pAUC}$~\citep{NarasimhanA-KDD13} with $\alpha=0$ and $\beta=1/n$, AATP~\citep{BoydCMR12} and InfinitePush~\citep{Agarwal11}. In addition, since the bipartite ranking problem can be solved as a binary classification problem, logistic regression (LR) which is shown to be consistent with bipartite ranking~\citep{KotlowskiDH11} and cost-sensitive SVM (cs-SVM) that addresses imbalance class distribution by introducing different misclassification costs are compared. Also, for completeness, SVM$^{\rm Rank}$~\citep{Joachims:2006} for AUC optimization are included in the comparison. We implement TopPush and InfinitePush using MATLAB, implement AATP using CVX~\citep{cvx} as in~\citep{BoydCMR12}, and use LIBLINEAR~\citep{Fan:2008} for LR and cs-SVM, and use the codes shared by the authors of the original works for other algorithms. It should be noted that binary classification algorithms LR and cs-SVM implemented by LIBLINEAR are of state-of-the-art efficiency.

In experiments, we measure the accuracy at the top of the ranked list by several commonly used metrics: (i)
positives at the top (Pos@Top)~\citep{Agarwal11,Rakotomamonjy12,BoydCMR12}, which is defined as the fraction of positive instances ranked above the top-ranked negative instance, (ii) average precision (AP) and (iii) normalized DCG scores (NDCG). In addition, ranking performance in terms of AUC are also reported.

On each dataset, experiments are run for thirty trials. In each trial, the dataset is randomly divided into two subsets: 2/3 for training and 1/3 for test. For all algorithms in comparison, we set the precision parameter $\epsilon$ to $10^{-4}$, choose other parameters by a 5-fold cross validation (based on the average value of Pos@Top) on training set, and evaluate the performance on test set. In detail, the regularization parameter $\lambda$ or $C$ is chosen from $\{10^{-3}, 10^{-2}, \ldots, 10^3\}$. For cs-SVM, the misclassification cost for positive instances is chosen from $\{10^{-3}, 10^{-2}, \ldots, 10^3\}$. For AATP, the parameter $\tau$ is from $\{2^{-5}, 2^{-4}, \ldots, 1\} \times \frac{m}{m+n}$, where $m$ and $n$ are the number of positive and negative instances respectively. The intervals are extended if the best parameter is on the boundary. Finally, averaged results over thirty trails are reported.  All experiments are run on a workstation with two Intel Xeon E7 CPUs and 16G memory.

\begin{table}[!t]
\renewcommand\arraystretch{0.72}
\centering \footnotesize
\begin{tabular}{l@{~~}l@{~~}|@{~~}l@{~~}l@{~~}|@{~~}l@{~~}l@{~~}|@{~~}l}
\toprule
\textbf{Data} & \textbf{Algorithm} &   \textbf{Time} (s) &  \textbf{Pos@Top} & \quad~\textbf{AP} &  ~~\textbf{NDCG} & \quad\textbf{AUC}\\
\midrule
\texttt{diabetes}
& {TopPush} & 5.11$\times$10$^{-3}$ &
.123$\pm$.056 & .872$\pm$.023 &.976$\pm$.005 & .780$\pm.$037\\
~ {\scriptsize 500/268}
& LR & 2.30$\times$10$^{-2}$ &
.064$\pm$.075$\bullet$ & .881$\pm$.022 &.973$\pm$.008 & .810$\pm$.030$\circ$\\
~ {\scriptsize d: 34}
& cs-SVM  & 7.70$\times$10$^{-2}$ &
.077$\pm$.088$\bullet$ & .758$\pm$.166$\bullet$ &.920$\pm$.078$\bullet$ & .624$\pm$.246$\bullet$\\
& SVM$^{\rm Rank}$  &  6.11$\times$10$^{-2}$ &
.087$\pm$.082$\bullet$ & .879$\pm$.022 & .975$\pm$.006 & .801$\pm$.033$\circ$\\
& SVM$^{\rm MAP}$   & 4.71$\times$10$^{0}$ &
.077$\pm$.072$\bullet$ & .879$\pm$.012 & .969$\pm$.009 & .616$\pm$.191$\bullet$\\
& SVM$^{\rm pAUC}$  & 2.09$\times$10$^{-1}\blacktriangle$ &
.053$\pm$.096$\bullet$ & .668$\pm$.123$\bullet$ & .884$\pm$.065$\bullet$ & .506$\pm$.167$\bullet$\\
& InfinitePush      &  2.63$\times$10$^{1}\bigstar$ &
.119$\pm$.051 & .877$\pm$.035 &.978$\pm$.007 & .793$\pm$.041\\
& AATP      &  2.72$\times$10$^{3}\bigstar$ &
.127$\pm$.061 & .881$\pm$.035 &.979$\pm$.010 & .783$\pm$.038 \\
\midrule
\texttt{news20-forsale}
& {TopPush} & 2.16$\times$10$^0$ &
.191$\pm$.088 & .843$\pm$.018 &.970$\pm$.005 & .969$\pm$.005 \\
~ {\scriptsize 999/18,929}
& LR & 4.14$\times$10$^0$ &
.086$\pm$.067$\bullet$ & .803$\pm$.020$\bullet$ &.962$\pm$.005 & .973$\pm$.004\\
~ {\scriptsize d: 62,061}
& cs-SVM  & 1.89$\times$10$^0$  &
.114$\pm$.069$\bullet$ & .766$\pm$.021$\bullet$ &.955$\pm$.006$\bullet$ & .964$\pm$.005\\
& SVM$^{\rm Rank}$ &  2.96$\times$10$^2\bigstar$ &
.149$\pm$.056$\bullet$  & .850$\pm$.016 &  .972$\pm$.003  & .974$\pm$.004\\
& SVM$^{\rm MAP}$   &  8.42$\times$10$^2\bigstar$ & 
.184$\pm$.092 & .832$\pm$.022 & .969$\pm$.007 & .961$\pm$.008\\
& SVM$^{\rm pAUC}$  & 3.25$\times$10$^2\bigstar$ & 
.196$\pm$.087 & .812$\pm$.019$\bullet$ & .963$\pm$.005$\bullet$ & .957$\pm$.007$\bullet$ \\
\midrule
\texttt{nslkdd}
& {TopPush} & 7.64$\times$10$^1$ &
.633$\pm$.088 & .978$\pm$.001 &.997$\pm$.001 & .969$\pm$.003\\
~ {\scriptsize 71,463/77,054}
& LR  & 3.63$\times$10$^1$ &
.220$\pm$.053$\bullet$ & .981$\pm$.002 &.998$\pm$.001 & .972$\pm$.002\\
~ {\scriptsize d: 121}
& cs-SVM  & 1.86$\times$10$^0$ &
.556$\pm$.037$\bullet$ & .980$\pm$.001 &.998$\pm$.001 & .972$\pm$.001 \\
& SVM$^{\rm pAUC}$   & 1.72$\times$10$^2$ & 
.634$\pm$.059 & .956$\pm$.002$\bullet$ & .996$\pm$.001 & .948$\pm$.002$\bullet$ \\
\midrule
\texttt{real-sim}
& {TopPush} &  1.34$\times$10$^1$ &
.186$\pm$.049 & .986$\pm$.001 &.998$\pm$.001 & .992$\pm$.002 \\
~ {\scriptsize 22,238/50,071}
& LR  &  7.67$\times$10$^0$　&
.100$\pm$.043$\bullet$ & .989$\pm$.001 &.999$\pm$.001 & .995$\pm$.002　\\
~ {\scriptsize d: 20,958}
& cs-SVM   & 4.84$\times$10$^0$  &
.146$\pm$.031$\bullet$ & .979$\pm$.001 &.998$\pm$.001  & .989$\pm$.001　\\
& SVM$^{\rm Rank}$  & 1.83$\times$10$^3\bigstar$ &  
.090$\pm$.045$\bullet$ &  .986$\pm$.000 & .999$\pm$.001 & .994$\pm$.002\\
\midrule
\texttt{spambase}
& {TopPush} & 1.51$\times$10$^{-1}$ &
.129$\pm$.077 & .922$\pm$.006 &.988$\pm$.001  & .942$\pm$.005 \\
~ {\scriptsize 1,813/2,788}
& LR  & 3.11$\times$10$^{-2}$ &
 .071$\pm$.053$\bullet$ & .920$\pm$.010 &.987$\pm$.003 & .952$\pm$.005$\circ$ \\
~ {\scriptsize d: 57}
& cs-SVM & 8.31$\times$10$^{-2}$ &
.069$\pm$.059$\bullet$ & .907$\pm$.010$\bullet$ & .980$\pm$.004$\bullet$ & .941$\pm$.005\\
& SVM$^{\rm Rank}$   &  2.31$\times$10$^{1}\blacktriangle$ & 
.069$\pm$.076$\bullet$ &  .931$\pm$.010  & .990$\pm$.003 & .970$\pm$.005$\circ$ \\
& SVM$^{\rm MAP}$   & 1.92$\times$10$^{2}\bigstar$ &
.097$\pm$.069$\bullet$ &  .935$\pm$.014  & .984$\pm$.005 & .920$\pm$.007 \\
& SVM$^{\rm pAUC}$  & 1.73$\times$10$^{0}\blacktriangle$  &
.073$\pm$.058$\bullet$ & .854$\pm$.024$\bullet$ & .975$\pm$.007$\bullet$ & .889$\pm$.019$\bullet$ \\
& InfinitePush   & 1.78$\times$10$^{3}\bigstar$ &
.132$\pm$.087 & .920$\pm$.005 &.987$\pm$.002 & .947$\pm$.007 \\
\midrule
\texttt{url}
& {TopPush}  & 5.11$\times$10$^3$ &
.474$\pm$.046 & .986$\pm$.001 &.999$\pm$.001 & .988$\pm$.002  \\
~ {\scriptsize 792,145/1,603,985}
& LR & 8.98$\times$10$^3$  &
.362$\pm$.113$\bullet$ & .993$\pm$.001$\circ$ &.999$\pm$.001 & .992$\pm$.002\\
~ {\scriptsize d: 3,231,961}
& cs-SVM  &  3.78$\times$10$^3$ &
.432$\pm$.069$\bullet$ & .991$\pm$.002 & .998$\pm$.001 & .998$\pm$.001 \\
\midrule
\texttt{w8a}
& {TopPush} & 7.35$\times$10$^0$ &
 .226$\pm$.053 & .710$\pm$.019 &.938$\pm$.005 & .922$\pm$.008 \\
~ {\scriptsize 1,933/62,767}
& LR  & 2.46$\times$10$^0$ &
.107$\pm$.093$\bullet$ & .450$\pm$.374$\bullet$ &.775$\pm$.221$\bullet$ & .591$\pm$.460$\bullet$\\
~ {\scriptsize d: 300}
& cs-SVM   &  3.87$\times$10$^0$  &
.118$\pm$.105$\bullet$ & .447$\pm$.372$\bullet$ &.774$\pm$.220$\bullet$ & .591$\pm$.461$\bullet$ \\
& SVM$^{\rm pAUC}$  & 2.59$\times$10$^3\bigstar$ &
.207$\pm$.046 & .673$\pm$.021$\bullet$ & .929$\pm$.006$\bullet$ & .911$\pm$.010$\bullet$ \\
\bottomrule
\end{tabular}
\caption{\small Data statistics (left column) and experimental results. For each dataset,
the number of positive and negative instances is below the data name as $m/n$, together with the number of dimensions $d$.
For training time comparison,``$\blacktriangle$'' (``$\bigstar$'') are marked if TopPush is at least 10 (100)
times faster than the compared algorithm. For performance (mean$\pm$std) comparison, ``$\bullet$'' (``$\circ$'') are marked
if TopPush performs significantly better (worse) than the baseline method based on pairwise $t$-test at 0.9 significance
level. On each dataset, if the evaluation of an algorithm can not be completed in two weeks,
it will be stopped and the corresponding results will be missing from the table.
}\label{tbl:data}
\end{table}

\subsection{Results}
In Table~\ref{tbl:data}, we report the performance of the algorithms in comparison, where the statistics of testbeds are included in the first column of the table. For better comparison between the performance of TopPush and baselines, pairwise $t$-tests at the significance level of $0.9$ are performed and results are marks ``$\bullet$ / $\circ$'' in Table~\ref{tbl:data} when they are statistically significantly worse/better than TopPush.
When an evaluation task that evaluates one algorithm on a dataset, including parameter selection, training and testing, can not be completed in two weeks, it will be stopped automatically, and no result will be reported. This is why some algorithms are missing from the table for certain datasets, especially for those large datasets.

We can see from Table~\ref{tbl:data} that TopPush, LR and cs-SVM succeed to finish the evaluation on all datasets (even the largest datasets \texttt{url}). In contrast, SVM$^{\rm Rank}$, SVM$^{\rm Rank}$ and SVM$^{\rm pAUC}$ fail to complete the task in time for several large datasets. InfinitePush and AATP have the worst scalability: they are only able to finish the smallest dataset \texttt{diabetes}, this is easy to understand since InfinitePush needs to solve an optimization problem with $mn$ variables and AATP needs to solve $m+n$ quadratic program problems. We thus find that overall, the proposed TopPush algorithm scales well to large datasets.

\subsubsection{Ranking Performance}
In terms of evaluation metric Pos@Top, we find that TopPush yields similar performance as InfinitePush and AATP, and performs significantly better than the other baselines including LR and cs-SVM, SVM$^{\rm Rank}$, SVM$^{\rm MAP}$ and SVM$^{\rm pAUC}$. This is consistent with the design of TopPush that aims to maximize the accuracy at the top of the ranked list. Since the loss function optimized by InfinitePush and AATP are similar as that for TopPush, it is not surprising that they yield similar performance. The key advantage of using the proposed algorithm versus InfinitePush and AATP is that it is computationally more efficient and scales well to large datasets. In terms of AP and NDCG, we observe that TopPush yield similar, if not better, performance as the state-of-the-art methods, such as SVM$^{\rm MAP}$ and SVM$^{\rm pAUC}$, that are designed to optimize these metrics. Overall, we can conclude that TopPush is effective in optimizing the ranking accuracy for the top ranked instances.

Meanwhile, we can see that TopPush achieves similar AUC values with on most datasets (only worse than SVM$^{\rm Rank}$ that is specially designed for AUC optimization on three datasets, but their differences are not large). This can be understood by Proposition~\ref{thm:loss}, which shows that the loss function~(\ref{eq:ipl2}) is a upper bound over the ranking loss, and TopPush which minimizes~(\ref{eq:ipl2}) can also achieve a small ranking loss and hereafter a good AUC.

\subsubsection{Training Efficiency }
To evaluate computational efficiency, we set the parameters of different algorithms to be the values that are selected by cross-validation, and run these algorithms on full datasets that include both training and testing sets. Table~\ref{tbl:data} summarizes the training time of different algorithms. From the results, we can see that TopPush is faster than state-of-the-art ranking methods on most datasets. In fact, the training time of TopPush is even similar to that of LR and cs-SVM implemented by LIBLINEAR. Since the time complexity of learning a binary classification model is usually linear in the number of training instances, this result implicitly suggests a linear time complexity for the proposed algorithm.

\begin{figure}
\centering
\includegraphics[width=0.55\linewidth]{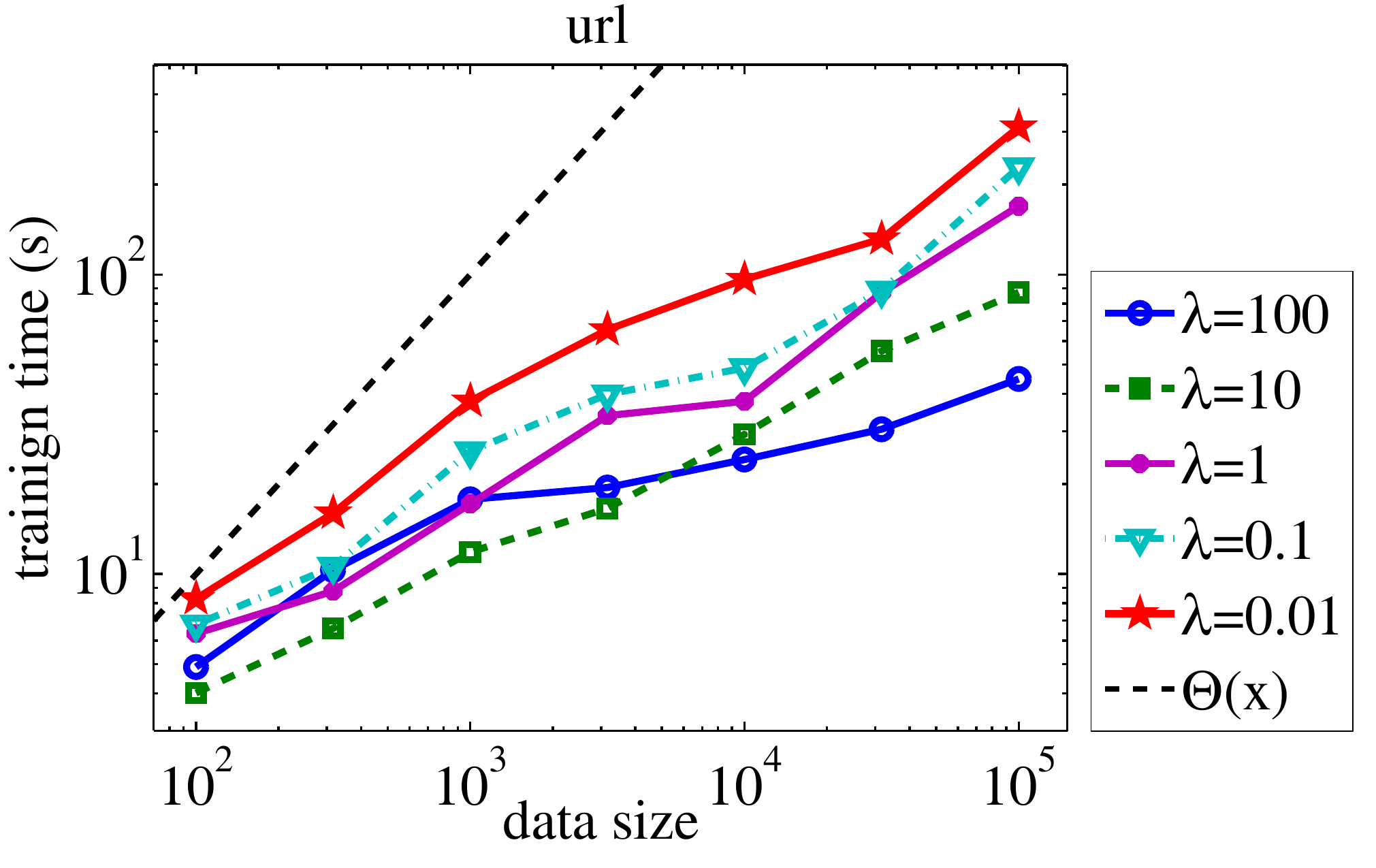}\vspace{-3mm}
\caption{Training time of TopPush versus training data size for different values of $\lambda$.}\label{fig:scale}
\end{figure}

\subsubsection{Scalability}
We study how TopPush scales to different number of training examples by using the largest dataset {\tt url}. Figure~\ref{fig:scale} shows the log-log plot for the training time of TopPush vs. the size of training data, where different lines correspond to different values of $\lambda$.
Lines in a log-log plot correspond to polynomial growth $\Theta(x^p)$, where $p$ corresponds to the slope of the line. %
For the purpose of comparison, we also include a black dash-dot line that tries to fit the training time by a linear function in the number of training instances (i.e., $\Theta(m+n)$). From the plot, we can see that for different regularization parameter $\lambda$, the training time of TopPush increases even slower than the number of training data.  This is consistent with our theoretical analysis given in Section~\ref{sec:complexity}.

\subsubsection{Influence of Parameters}
We study the influence of precision parameter $\epsilon$ and regularization parameter $\lambda$ on the computational cost and prediction performance of TopPush.
First, we fix $\lambda$ to be $1$, and run TopPush with $\epsilon \in \{10^{-8}, \ldots, 10^{-2}\}$. We measure the number of iterations needed to achieve the accuracy $\epsilon$, and the prediction performance of the learned ranking function. Figure~\ref{fig:epsilon} show the results for dataset {\tt w8a}. Similar results are obtained for the other datasets. It is not surprising to observe that the smaller the $\epsilon$, the better the prediction performance, but at the price of a larger number of iterations and consequentially a longer training time. Evidently, we may want to set the precision parameter $\epsilon$ to balance the tradeoff between computational time and prediction performance. According to Figure~\ref{fig:epsilon}, we found that $\epsilon=10^{-4}$ appears to achieve nearly optimal performance with a small number of iterations.

\begin{figure}[!t]
  \centering
  \includegraphics[height=0.3\linewidth]{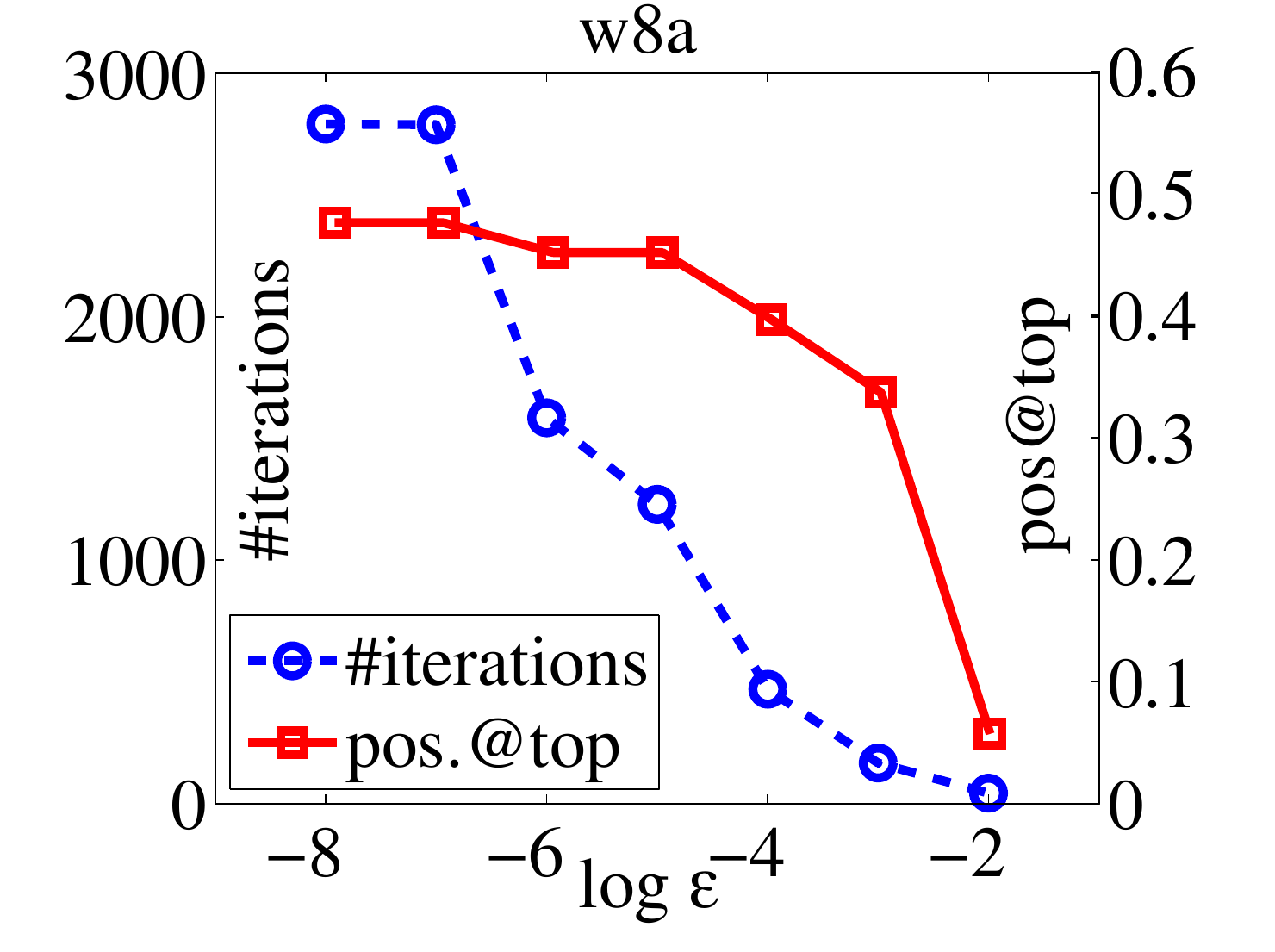} \quad\quad
  \includegraphics[height=0.3\linewidth]{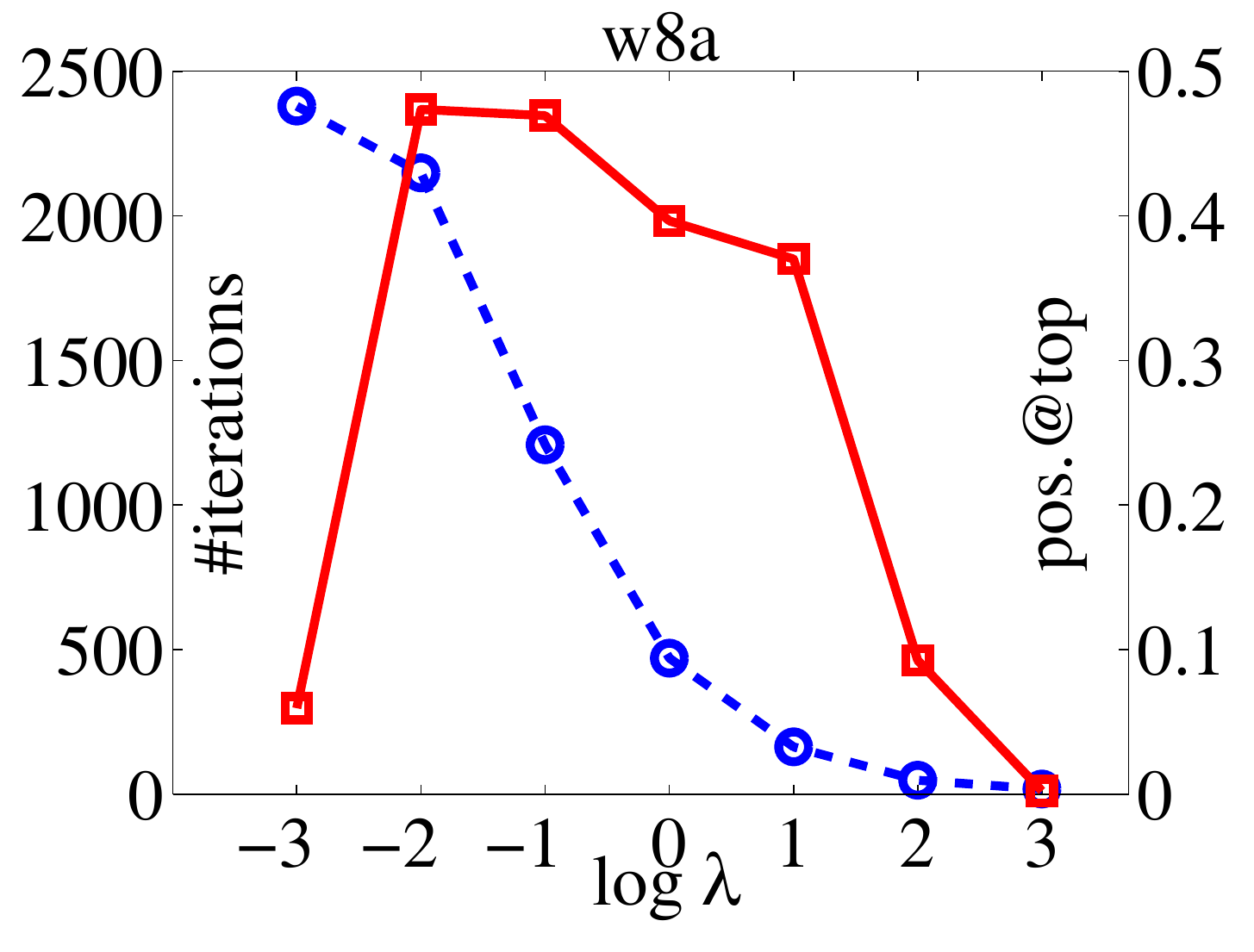}
  \caption{Influence of the precision parameter $\epsilon$ and the regularization parameter $\lambda$ on TopPush, where the horizontal axis is $\epsilon$ and $\lambda$,  vertical axes are number of iterations (left) and prediction performance (right), legends of two plots are the same. }\label{fig:epsilon}
\end{figure}
In the second experiment, we fix $\epsilon$ to $10^{-4}$, and examine the influence of $\lambda$. Figure~\ref{fig:epsilon} shows how the number of iterations and prediction accuracy are affected by different $\lambda$ on dataset {\tt w8a}. We observe that the smaller the $\lambda$, the smaller the number of iterations. This is because regularization parameter $\lambda$ controls the domain size, and as a result, a smaller $\lambda$ will lead to a smaller solution domain and thus a faster convergence to the optimal solution.
As expected, we need to choose the value $\lambda$ to achieve good performance, since it is a regularization parameter. Meanwhile, the computational cost of TopPush reduces when a larger value of $\lambda$ is used. This is easy to understand, because $\lambda$ controls the size of the domain from which TopPush searches the optimal ranking function, and a large $\lambda$ reduces the domain size. Empirically, we can set $\lambda$ to 1 by default, and search  $\lambda$ in $\{10^{-2}, \ldots, 10^{2}\}$ for better solution.

\section{Conclusion and Future Work}\label{sec:con}
In this paper, we focus on bipartite ranking algorithms that optimize accuracy at the top of the ranked list.
To this end, we consider to maximize the number of positive instances that are ranked above any negative instances, and develop an efficient algorithm, named as {TopPush} to solve related optimization problem. Compared with existing work on this topic, the proposed {TopPush} algorithm scales linearly in the number of training instances, which is in contrast to most existing algorithms for bipartite ranking whose time complexities dependents on the number of positive-negative instance pairs. Moreover, our theoretical analysis clearly shows that it will lead to a ranking function that places many positive instances the top of the ranked list. Empirical studies verify the theoretical claims: the TopPush algorithm is effective in maximizing the accuracy at the top and is significantly more efficient than the state-of-the-art algorithms for bipartite ranking. In the future, we plan to develop appropriate univariate loss, instead of pairwise ranking loss, for efficient bipartite ranking that maximize accuracy at the top.

\section*{Acknowledgments}
This research was supported by the 973 Program (2014CB340501), NSFC (61333014), NSF (IIS-1251031), and ONR Award (N000141210431). Z.-H. Zhou is the corresponding author of this paper.

\vskip 0.2in
\bibliography{topPush-full}

\bibliographystyle{plain}

\end{document}